\documentclass[conference]{IEEEtran}
\IEEEoverridecommandlockouts

\usepackage{cite}
\usepackage{url}
\usepackage{amsmath,amssymb,amsfonts}
\usepackage{algorithmic}
\usepackage{algorithm}
\usepackage{graphicx}
\usepackage{textcomp}
\usepackage{xcolor}
\usepackage{bm}

\newcommand{\E}{\mathbb{E}}
\newcommand{\Prb}{\mathbb{P}}
\newcommand{\N}{\mathbb{N}}
\newcommand{\eps}{\varepsilon}
\newcommand{\argmax}{\operatorname*{arg\,max}}
\newcommand{\UCB}{\operatorname{UCB}}
\newcommand{\LCB}{\operatorname{LCB}}
\newcommand{\tstrut}{\rule[-1.7ex]{0pt}{5.1ex}}

\newtheorem{theorem}{Theorem}
\newtheorem{lemma}[theorem]{Lemma}
\newtheorem{remark}[theorem]{Remark}

\newenvironment{proof}{\noindent\textit{Proof:}\,}{\hfill$\square$\medskip}

\begin{document}

\title{Improved Multiplayer Bandit Algorithms\\for Bernoulli Rewards}

\author{\IEEEauthorblockN{Khang Nguyen}
\IEEEauthorblockA{\textit{Dept.\ of Computer Science} \\
\textit{University of California, Los Angeles}\\
Los Angeles, CA, USA}
\and
\IEEEauthorblockN{Ricardo Parada}
\IEEEauthorblockA{\textit{Dept.\ of Computer Science} \\
\textit{University of California, Los Angeles}\\
Los Angeles, CA, USA}
\and
\IEEEauthorblockN{William Chang}
\IEEEauthorblockA{\textit{Dept.\ of Electrical and Computer Engineering} \\
\textit{University of California, Los Angeles}\\
Los Angeles, CA, USA}
}

\maketitle

\begin{abstract}
We study the multiplayer multi-armed bandit problem with information asymmetry under Bernoulli rewards, for three information structures: asymmetry in actions, in rewards, and in both. Replacing the Hoeffding-style confidence intervals of prior work with Kullback--Leibler (KL) divergence-based bounds gives strictly tighter regret guarantees in each case. We propose \texttt{mKL-UCB}, \texttt{mKL-UCB-Intervals} and \texttt{mKL-DSEE}, and show that the improvement factor is at least two by Pinsker's inequality and far larger when reward means are near zero or one. For asymmetry in rewards we prove that two arms' KL intervals separate after a deterministic number of samples, and that $M$ independent players accelerate elimination further.
\end{abstract}

\begin{IEEEkeywords}
Multiplayer bandits, information asymmetry, KL-UCB, regret analysis
\end{IEEEkeywords}

\section{Introduction}

The multi-armed bandit (MAB) problem is a foundational model for sequential decision-making under uncertainty \cite{lattimore2020bandit}, in which the optimal regret scales as $\Theta(\log T)$ with a constant determined by the Kullback--Leibler (KL) divergence between arm distributions \cite{lai1985asymptotically}. For Bernoulli rewards, KL-UCB \cite{garivier2011kl, cappe2013kullback} exploits Chernoff-type concentration for tighter finite-time bounds than generic sub-Gaussian analyses, replacing Hoeffding-style intervals of width $O(\sqrt{\log T / n})$ by KL intervals that are never wider and often much tighter.

Many modern applications involve several agents that must coordinate without direct communication, among them multi-stakeholder recommendation \cite{abdollahpouri2017multistakeholder} and decentralized network routing. Chang et al.\ \cite{chang2022online} introduced the \emph{information-asymmetric multiplayer bandit} framework for such settings, where $M$ players simultaneously select individual actions forming a joint action $\bm{a} = (a_1, \ldots, a_M)$ and the reward depends on this joint action. Players may or may not observe each other's actions, and may or may not receive the same reward realization; crucially, they cannot communicate during learning, though they may agree on a common strategy beforehand. A follow-up work \cite{chang2023optimal} refined the case of observable actions and independent rewards with an interval-based elimination strategy; extensions to metric and contextual settings appear in \cite{chang2025multiplayer, chang2025multiplayerContextual}. Unlike the cooperative bandit literature \cite{awerbuch2008competitive}, which typically permits some communication during play, this framework allows none.

Both \cite{chang2022online} and \cite{chang2023optimal} derive regret bounds under sub-Gaussian assumptions, which treat all bounded distributions alike: Hoeffding's inequality yields the same confidence width $\sqrt{2\log T / n}$ whether the reward is Bernoulli(0.01) or Bernoulli(0.5). Yet in the applications above rewards are binary---a user clicks or does not, a packet is delivered or dropped---and success probabilities are far from $1/2$: click-through rates are typically $1\text{--}5\%$, delivery rates in congested networks often above $95\%$. The KL divergence between Bernoulli distributions with nearby means satisfies $D_{KL}(\mu+\eps, \mu) \approx \eps^2/(2\mu(1-\mu))$ for small $\eps$, whereas the sub-Gaussian exponent is proportional to $\eps^2$ with a constant that does not adapt to $\mu$. The improvement factor $\frac{1}{2\mu(1-\mu)}$ is $2$ at $\mu=1/2$ (the worst case) but grows to $50$ at $\mu=0.01$ and $500$ at $\mu=0.001$, and the multiplayer setting amplifies it: the joint action space has $K^M$ arms, so the sample budget available to each arm is thin.

\textbf{Contributions.} We give three algorithms. For \textbf{Problem A} (asymmetry in actions), \texttt{mKL-UCB}, whose leading term $\sum_{\bm{a}} \Delta_{\bm{a}} \log T / D_{KL}(\mu_{\bm{a}}, \mu^*)$ matches the Lai--Robbins lower bound; for \textbf{Problem B} (asymmetry in rewards), \texttt{mKL-UCB-Intervals}, with a complete analysis of when KL intervals separate; and for \textbf{Problem C} (asymmetry in both), \texttt{mKL-DSEE}, a phased algorithm with improved tail bounds. Throughout, the improvement amounts to replacing squared-error terms $\eps^2$ by KL divergence terms $D_{KL}(\mu \pm \eps, \mu)$.

\section{Problem Setting and Preliminaries}\label{sec:setting}

At each round $t = 1, \ldots, T$, each of $M$ players $P_1,\ldots,P_M$ selects $a_i \in [K] := \{1,\ldots,K\}$, forming the joint action $\bm{a} = (a_1,\ldots,a_M) \in [K]^M$, and a reward $X_{\bm{a}}(t) \sim \mathrm{Bernoulli}(\mu_{\bm{a}})$ is drawn with unknown mean $\mu_{\bm{a}} \in [0,1]$. Players may agree on a common strategy before play but cannot communicate during it. In \textbf{Problem A} all players observe the same reward but not each other's actions; in \textbf{Problem B} players observe actions but receive independent realizations $X^{(i)}_{\bm{a}}(t) \overset{\text{i.i.d.}}{\sim} \mathrm{Bernoulli}(\mu_{\bm{a}})$; in \textbf{Problem C} they observe neither actions nor a common reward.

Let $\mu^* = \max_{\bm{a}} \mu_{\bm{a}}$, $\bm{a}^* \in \argmax_{\bm{a}} \mu_{\bm{a}}$, $\Delta_{\bm{a}} = \mu^* - \mu_{\bm{a}}$, $\Delta_{\max} = \max_{\bm a}\Delta_{\bm a}$, and $n_{\bm{a}}(T) = \sum_{t=1}^T \mathbf{1}\{\bm{a}_t = \bm{a}\}$, so that the expected cumulative regret is $R_T = \sum_{\bm{a} \in [K]^M} \Delta_{\bm{a}} \E[n_{\bm{a}}(T)]$. All algorithms break ties by the lexicographic order on $[K]^M$: $\bm{x} < \bm{y}$ if $x_i = y_i$ for all $i < j$ and $x_j < y_j$ for some $j \in [M]$. The KL divergence between Bernoulli distributions is $D_{KL}(p, q) = p\log\frac{p}{q} + (1\!-\!p)\log\frac{1\!-\!p}{1\!-\!q}$.

\begin{lemma}[Chernoff bound for Bernoulli]\label{lem:chernoff}
Let $X_1, \ldots, X_n \overset{\text{i.i.d.}}{\sim} \mathrm{Bernoulli}(\mu)$ with sample mean $\hat{\mu}$. Then
$\Prb(\hat{\mu} \geq \mu + \eps) \leq e^{-n D_{KL}(\mu+\eps,\, \mu)}$
and
$\Prb(\hat{\mu} \leq \mu - \eps) \leq e^{-n D_{KL}(\mu-\eps,\, \mu)}$.
\end{lemma}

By Pinsker's inequality $D_{KL}(p,q) \geq 2(p-q)^2$, the Chernoff exponent is always at least $2n\eps^2$, recovering Hoeffding's bound $\Prb(\hat{\mu} \geq \mu+\eps) \leq e^{-2n\eps^2}$ (and symmetrically for the lower tail), strictly for $\mu \neq 1/2$. Our bounds also use a quantitative monotonicity property of the relative entropy \cite[Lemma~10.2]{lattimore2020bandit}.

\begin{lemma}[KL monotonicity]\label{lem:kl-mono}
Let $p, q \in [0,1]$ and $\eps \geq 0$. (a)~$D_{KL}(\cdot, q)$ is convex with unique minimum at $q$, decreasing on $[0,q]$ and increasing on $[q,1]$; likewise $D_{KL}(p, \cdot)$ is convex with unique minimum at $p$, decreasing on $[0,p]$ and increasing on $[p,1]$. (b)~(Pinsker) $D_{KL}(p, q) \geq 2(p-q)^2$. (c)~If $p \leq q-\eps \leq q$, then $D_{KL}(p, q-\eps) \leq D_{KL}(p,q) - D_{KL}(q-\eps,q) \leq D_{KL}(p,q) - 2\eps^2$, and symmetrically if $q \leq q+\eps \leq p$.
\end{lemma}

Property~(c) is what makes KL bounds quantitatively tighter: it converts a perturbation of the second argument into an additive $2\eps^2$ slack inside the KL functional, which can be exchanged with a Chernoff exponent. Inverting Lemma~\ref{lem:chernoff} through Lemma~\ref{lem:kl-mono}(a) yields valid intervals: for $\alpha \geq 0$ the set $\{u : D_{KL}(\hat{\mu},u) \leq \alpha\}$ contains $\hat\mu$ and its endpoints trap $\mu$ with probability at least $1-2e^{-n\alpha}$ \cite[Cor.~10.4]{lattimore2020bandit}.

Standard UCB uses $\hat{\mu}_a(t) + \sqrt{2\log T/n_a(t)}$, whose width does not depend on $\hat\mu_a(t)$, whereas the KL set $\{\tilde\mu : D_{KL}(\hat{\mu}_a(t), \tilde\mu) \leq c/n_a(t)\}$ narrows as $\hat{\mu}_a(t)$ approaches $0$ or $1$, where the distribution is most peaked, and always lies inside the sub-Gaussian interval.

\section{Problem A: Asymmetry in Actions}\label{sec:probA}

Define the KL-UCB index over joint arms,
\begin{equation}\label{eq:klucb}
    \UCB^{\mathrm{KL}}_{\bm{a}}(t) = \max\!\Big\{\tilde{\mu} \in [0,1] : D_{KL}(\hat{\mu}_{\bm{a}}(t), \tilde{\mu}) \leq \tfrac{\log f(t)}{n_{\bm{a}}(t)}\Big\},
\end{equation}
with $f(t) = 1 + t\log^2 t$. The threshold $\log f(t)$ is the simplest threshold of the form $\log t + o(\log t)$ that makes the bound hold uniformly in time and yields asymptotic optimality \cite{garivier2011kl}.

\begin{algorithm}[t]
\caption{\texttt{mKL-UCB} for Problem A}\label{alg:probA}
\small
\begin{algorithmic}[1]
\REQUIRE Players $M$, arms per player $K$, horizon $T$
\STATE \textbf{Init:} For $t=1,\ldots,K^M$, play joint actions in lex.\ order.
\FOR{$t = K^M+1, \ldots, T$}
\STATE Compute $\UCB^{\mathrm{KL}}_{\bm{a}}(t)$ for all $\bm{a} \in [K]^M$ via \eqref{eq:klucb}.
\STATE Play $\bm{a}^*(t) = \min\big\{\bm{a} : \UCB^{\mathrm{KL}}_{\bm{a}}(t) = \max_{\bm{a}'} \UCB^{\mathrm{KL}}_{\bm{a}'}(t)\big\}$.
\STATE Observe common reward and update statistics.
\ENDFOR
\end{algorithmic}
\end{algorithm}

\textbf{Where the asymmetry enters.} Player $i$ sees only its own component $a_i$, so a priori it cannot tell which joint arm produced the reward it observed, and without that attribution it cannot maintain empirical means indexed by joint arms. Algorithm~\ref{alg:probA} works because the shared reward makes the players' information identical: all start from the same statistics, see the same rewards, and apply the same deterministic index rule with the same lexicographic tie-break, so each reconstructs the whole joint action of every round. The asymmetry is neutralized rather than absent, and neutralizing it restricts the algorithm class: randomized tie-breaking, or private randomization such as Thompson sampling with unshared seeds, would decouple the reconstructions and coordination would fail. Given reconstruction, the problem reduces to single-agent KL-UCB over the $K^M$ joint arms, which is where the multiplayer aspect exacts its price.

\begin{theorem}\label{thm:probA}
The expected regret of Algorithm~\ref{alg:probA} satisfies
\begin{align}\label{eq:probA-bound}
R_T &\leq \sum_{\bm{a}: \Delta_{\bm{a}} > 0} \inf_{\substack{\eps_1, \eps_2 > 0 \\ \eps_1 + \eps_2 < \Delta_{\bm{a}}}} \Delta_{\bm{a}} \bigg( \frac{\log f(T)}{D_{KL}(\mu_{\bm{a}} + \eps_1, \mu^* - \eps_2)} \notag \\
&\quad + \frac{1}{2\eps_1^2} + \frac{2}{\eps_2^2} \bigg).
\end{align}
In particular, $\limsup_{T\to\infty} R_T / \log T \leq \sum_{\bm{a}} \Delta_{\bm{a}}/D_{KL}(\mu_{\bm{a}}, \mu^*)$, matching the Lai--Robbins lower bound \cite{lai1985asymptotically}.
\end{theorem}

The proof is an instance of \cite[Thm.~10.6]{lattimore2020bandit} lifted to $[K]^M$, with $\eps_1,\eps_2$ the tolerances of its $\tau$-$\kappa$ decomposition: $\E[\tau] \leq 2/\eps_2^2$ and $\E[\kappa] \leq \log f(T)/D_{KL}(\mu_{\bm{a}}+\eps_1, \mu^*-\eps_2) + 1/(2\eps_1^2)$ by \cite[Lem.~10.7, 10.8]{lattimore2020bandit}, both relying on Lemma~\ref{lem:kl-mono}(a,c); the asymptotic claim follows with $\eps_i = 1/\log\log T$. Algorithm~\ref{alg:probA} is thus asymptotically optimal among consistent algorithms, the sub-Gaussian constant $\sum_{\bm{a}} 2/\Delta_{\bm{a}}$ being strictly larger by Pinsker.

\section{Problem B: Asymmetry in Rewards}\label{sec:probB}

Players now observe each other's actions but receive independent rewards, so they hold different empirical means and may disagree on the best arm; KL-UCB with independent statistics then gives linear regret \cite{chang2022online}. The solution \cite{chang2023optimal} is \emph{confidence interval-based elimination}: players share a desired set $\mathcal{D}$ of candidates and drop an arm when any player finds its interval strictly dominated by another's, signalling this by deviating from the expected action, which all players observe. For each player $i$ and joint arm $\bm{a}$, let $\hat{\mu}^i_{\bm{a}}(t)$ be player $i$'s empirical mean from $n_{\bm{a}}(t)$ samples and put
\begin{equation}\label{eq:int-B}
S^i_{\bm{a}}(t) = \Big\{\tilde{\mu} \in [0,1] : D_{KL}(\hat{\mu}^i_{\bm{a}}(t), \tilde{\mu}) \leq \tfrac{2\log T}{n_{\bm{a}}(t)}\Big\}.
\end{equation}
Its endpoints $\LCB^{i,\mathrm{KL}}_{\bm{a}}(t) = \min S^i_{\bm{a}}(t)$ and $\UCB^{i,\mathrm{KL}}_{\bm{a}}(t) = \max S^i_{\bm{a}}(t)$ define the interval $I^i_{\bm{a}}(t)$, and by Lemma~\ref{lem:chernoff} with a union bound, $\Prb(\mu_{\bm{a}} \notin I^i_{\bm{a}}(t)) \leq 2T^{-2}$.

\begin{remark}\label{rem:threshold}
The fixed threshold $2\log T$ is essential. An anytime threshold such as $\log(1+t\log^2 t)$ can briefly exclude the true mean early on; in Problem A this is harmless, since no arm is permanently discarded, but in Problem B elimination is permanent and observed by all, so one erroneous elimination of $\bm{a}^*$ makes every later round incur $\Delta_{\max}$ regret.
\end{remark}

\begin{lemma}[Disjointness time of two KL intervals]\label{lem:disjoint-pair}
Fix a player $i$, arms $\bm{a}, \bm{a}^*$ with $\mu_{\bm{a}} < \mu^*$, and $\eps_1, \eps_2 > 0$ with $\eps_1 + \eps_2 < \mu^* - \mu_{\bm{a}}$. Set $\tau' := \mu^* - \eps_2$ and
\begin{equation}\label{eq:disjoint-thresholds}
n_0 := \!\left\lceil\!\frac{2 \log T}{D_{KL}(\mu_{\bm{a}} + \eps_1,\, \tau')}\!\right\rceil,\quad n_0^* := \!\left\lceil\!\frac{2 \log T}{D_{KL}(\mu^*,\, \tau')}\!\right\rceil.
\end{equation}
On the concentration event $\mathcal{C}^i(t) := \{\hat{\mu}^i_{\bm{a}}(t) \leq \mu_{\bm{a}} + \eps_1\} \cap \{\hat{\mu}^i_{\bm{a}^*}(t) \geq \mu^*\}$,
\begin{equation}\label{eq:disjoint-deterministic}
\begin{aligned}
&n_{\bm{a}}(t) \geq n_0,\quad n_{\bm{a}^*}(t) \geq n_0^* \\
&\qquad\Longrightarrow\; \UCB^{i,\mathrm{KL}}_{\bm{a}}(t) \leq \tau' \leq \LCB^{i,\mathrm{KL}}_{\bm{a}^*}(t).
\end{aligned}
\end{equation}
The intervals at player $i$ are then disjoint and the elimination criterion of Algorithm~\ref{alg:probB} fires (replacing $n_0, n_0^*$ by $n_0+1, n_0^*+1$ makes it strict).
\end{lemma}

\begin{proof}
\emph{Suboptimal side.} On $\mathcal{C}^i(t)$ we have $\hat{\mu}^i_{\bm{a}}(t) \leq \mu_{\bm{a}} + \eps_1 \leq \tau' - \eps_2 < \tau'$ by $\eps_1+\eps_2 < \Delta_{\bm{a}}$, so $D_{KL}(\cdot, \tau')$ being decreasing on $[0,\tau']$ (Lemma~\ref{lem:kl-mono}(a)) gives $D_{KL}(\hat{\mu}^i_{\bm{a}}(t), \tau') \geq D_{KL}(\mu_{\bm{a}}+\eps_1, \tau')$. With $n_{\bm{a}}(t) \geq n_0$ and \eqref{eq:disjoint-thresholds},
\begin{equation}\label{eq:sub-bound}
n_{\bm{a}}(t)\, D_{KL}(\hat{\mu}^i_{\bm{a}}(t), \tau') \geq n_0\, D_{KL}(\mu_{\bm{a}}+\eps_1, \tau') \geq 2 \log T,
\end{equation}
so $\tau'$ lies outside the open ball $\{u : D_{KL}(\hat{\mu}^i_{\bm{a}}(t), u) < 2\log T/n_{\bm{a}}(t)\}$. Since $D_{KL}(\hat{\mu}^i_{\bm{a}}(t), \cdot)$ is increasing on $[\hat{\mu}^i_{\bm{a}}(t), 1]$, that ball meets $[\hat{\mu}^i_{\bm{a}}(t),1]$ in an interval with right endpoint $\UCB^{i,\mathrm{KL}}_{\bm{a}}(t)$, whence $\UCB^{i,\mathrm{KL}}_{\bm{a}}(t) \leq \tau'$.

\emph{Optimal side.} On $\mathcal{C}^i(t)$, $\hat{\mu}^i_{\bm{a}^*}(t) \geq \mu^* > \tau'$ and $D_{KL}(\cdot,\tau')$ is increasing on $[\tau',1]$, so $D_{KL}(\hat{\mu}^i_{\bm{a}^*}(t), \tau') \geq D_{KL}(\mu^*, \tau')$; \eqref{eq:sub-bound} with $n_0^*$ for $n_0$ puts $\tau'$ outside the ball at $\bm{a}^*$, and symmetrically $\LCB^{i,\mathrm{KL}}_{\bm{a}^*}(t) \geq \tau'$.
\end{proof}

\begin{remark}[Comparison with Hoeffding intervals]\label{rem:disjoint-vs-hoeff}
For Hoeffding intervals of half-width $\sqrt{2\log T/n}$ the analogous thresholds are $\lceil \log T/\eps_1^2\rceil$ and $\lceil \log T/\eps_2^2\rceil$; by Pinsker they dominate $n_0$ and $n_0^*$ whenever $\mu_{\bm{a}}+\eps_1+\eps_2 \leq \mu^*$. At $\mu_{\bm{a}} = 0.05$, $\mu^* = 0.15$, $\eps_1 = \eps_2 = 0.02$, $T = 10^5$ we get $n_0 \approx 535$ against $n_0^{\mathrm{Hoeff}} \approx 28{,}800$; this $\sim\!54\times$ ratio is why the KL intervals separate within $O(10^3)$ samples in our experiments while the Hoeffding ones do not separate within the horizon.
\end{remark}

\begin{algorithm}[t]
\caption{\texttt{mKL-UCB-Intervals} for Problem B}\label{alg:probB}
\small
\begin{algorithmic}[1]
\REQUIRE Players $M$, arms per player $K$, horizon $T$
\STATE Initialize desired set $\mathcal{D} \gets [K]^M$.
\FOR{$t = 1, \ldots, T$}
  \STATE Let $\bm{c}_t$ = next arm in $\mathcal{D}$ under agreed ordering.
  \IF{$\exists$ player $i$ and arm $\bm{b} \in \mathcal{D}$: $\LCB^{i,\mathrm{KL}}_{\bm{b}}(t) > \UCB^{i,\mathrm{KL}}_{\bm{c}_t}(t)$}
    \STATE Player $i$ deviates: plays any action $\neq \bm{c}_t[i]$.
  \ELSE
    \STATE All players play their component of $\bm{c}_t$.
  \ENDIF
  \STATE All players observe joint action $\bm{a}_t$; if $\bm{a}_t \neq \bm{c}_t$, all remove $\bm{c}_t$ from $\mathcal{D}$.
  \STATE Player $i$ observes $X^{(i)}_{\bm{a}_t}(t)$ and updates intervals.
\ENDFOR
\end{algorithmic}
\end{algorithm}

Since the thresholds of Lemma~\ref{lem:disjoint-pair} are uniformly smaller than their Hoeffding analogues, what remains is to show that $\mathcal{C}^i(t)$ holds at some player after few additional samples; this is where the multiplayer structure pays off.

\begin{lemma}[$M$-player suboptimal-arm count]\label{lem:disjoint-Mfold}
Fix a suboptimal $\bm{a}$ and $\eps_1, \eps_2 > 0$ with $\eps_1 + \eps_2 < \Delta_{\bm{a}}$, and set $\tau' := \mu^* - \eps_2$. Let $\hat{\mu}^i_{\bm{a},s}$ be player $i$'s empirical mean of $\bm{a}$ from $s$ samples, $\kappa_i := \sum_{s=1}^{T} \mathbf{1}\{D_{KL}(\hat{\mu}^i_{\bm{a},s}, \tau') \leq \tfrac{2\log T}{s}\}$ and $\bar{\kappa} := \min_{i \in [M]} \kappa_i$. Then
\begin{equation}\label{eq:Mfold-bound}
\E[\bar{\kappa}] \;\leq\; \frac{2\log T}{D_{KL}(\mu_{\bm{a}}+\eps_1,\, \tau')} + \frac{1}{M\, D_{KL}(\mu_{\bm{a}}+\eps_1,\, \mu_{\bm{a}})}.
\end{equation}
\end{lemma}

\begin{proof}
Set $n_0 := \lceil 2\log T / D_{KL}(\mu_{\bm{a}}+\eps_1, \tau')\rceil$. By Lemma~\ref{lem:kl-mono}(a), for $s \geq n_0$ the event $\{D_{KL}(\hat{\mu}^i_{\bm{a},s}, \tau') \leq 2\log T/s\}$ implies $\hat{\mu}^i_{\bm{a},s} > \mu_{\bm{a}} + \eps_1$, since otherwise $D_{KL}(\hat{\mu}^i_{\bm{a},s}, \tau') \geq D_{KL}(\mu_{\bm{a}}+\eps_1, \tau') \geq 2\log T/n_0 \geq 2\log T/s$. Hence $\bar{\kappa} \leq n_0 + \sum_{s=n_0}^{T} \mathbf{1}\{\forall i:\, \hat{\mu}^i_{\bm{a},s} > \mu_{\bm{a}}+\eps_1\}$ and
\begin{align*}
\E[\bar{\kappa}]
&\leq n_0 + \sum_{s=n_0}^{\infty} \Prb\!\big(\forall i:\, \hat{\mu}^i_{\bm{a},s} > \mu_{\bm{a}}+\eps_1\big) \\
&\stackrel{(\star)}{=} n_0 + \sum_{s=n_0}^{\infty} \prod_{i=1}^{M} \Prb\!\big(\hat{\mu}^i_{\bm{a},s} > \mu_{\bm{a}}+\eps_1\big) \\
&\stackrel{(\dagger)}{\leq} n_0 + \sum_{s=n_0}^{\infty} e^{-s M D_{KL}(\mu_{\bm{a}}+\eps_1,\, \mu_{\bm{a}})} \\
&\stackrel{(\ddagger)}{\leq} n_0 + \frac{1}{M\, D_{KL}(\mu_{\bm{a}}+\eps_1,\, \mu_{\bm{a}})}.
\end{align*}
Step~$(\star)$ uses independence of the players' reward streams, $(\dagger)$ applies Lemma~\ref{lem:chernoff} to each factor, and $(\ddagger)$ sums the geometric series via $\sum_{s\geq 0} e^{-sx} \leq 1/x + 1$, absorbing the $+1$ into $n_0$.
\end{proof}

Step~$(\star)$ is where $M$ enters the exponent rather than the prefactor: a union bound over players, as in the sub-Gaussian analysis, would give $Mn_0 + 1/(2\eps_1^2)$ and no gain over $M=1$.

\begin{theorem}\label{thm:probB}
For known horizon $T$, the expected regret of Algorithm~\ref{alg:probB} satisfies
\begin{align}\label{eq:probB-bound}
R_T &\leq \sum_{\bm{a}: \Delta_{\bm{a}} > 0} \inf_{\substack{\eps_1, \eps_2 > 0,\, \delta' \in (0, \eps_2) \\ \eps_1 + \eps_2 < \Delta_{\bm{a}}}} \Delta_{\bm{a}} \bigg( \frac{2\log T}{D_{KL}(\mu_{\bm{a}} + \eps_1, \mu^* - \eps_2)} \notag \\
&\quad + \frac{1}{M\, D_{KL}(\mu_{\bm{a}} + \eps_1, \mu_{\bm{a}})} + \frac{2\log T}{D_{KL}(\mu^* - \delta', \mu^* - \eps_2)} \notag \\
&\quad + \frac{1}{2(\delta')^2}\bigg) + 2 M K^M \Delta_{\max}.
\end{align}
As $\eps_1, \eps_2 \to 0$ with $\delta'/\eps_2 \to 0$ the leading constant is $\sum_{\bm{a}} 2\Delta_{\bm{a}}/D_{KL}(\mu_{\bm{a}}, \mu^*)$.
\end{theorem}

\begin{proof}
Fix a suboptimal $\bm{a}$, tolerances $\eps_1, \eps_2>0$ with $\eps_1+\eps_2 < \Delta_{\bm{a}}$, $\delta' \in (0,\eps_2)$, and $\tau' = \mu^*-\eps_2$.

\emph{Step 1 (Good event).} Let $\mathcal{G} = \{\mu_{\bm{b}} \in I^i_{\bm{b}}(t)\ \forall i, \bm{b}, t \leq T\}$; a union bound gives $\Prb(\mathcal{G}^c) \leq 2MK^M/T$, so $\mathcal{G}^c$ contributes at most $2MK^M\Delta_{\max}$ to $\E[R_T]$, the additive term of \eqref{eq:probB-bound}. On $\mathcal{G}$, $\LCB^{i,\mathrm{KL}}_{\bm{b}} \leq \mu_{\bm{b}} < \mu^* \leq \UCB^{i,\mathrm{KL}}_{\bm{a}^*}$ for every suboptimal $\bm{b}$ and every player, so $\bm{a}^*$ is never eliminated. We work on $\mathcal{G}$ below.

\emph{Step 2 (Decomposition).} By line~4 of Algorithm~\ref{alg:probB}, $\bm{a}$ is eliminated once one player $i$ has $\LCB^{i,\mathrm{KL}}_{\bm{a}^*}(t) > \UCB^{i,\mathrm{KL}}_{\bm{a}}(t)$, which Lemma~\ref{lem:disjoint-pair} guarantees on $\mathcal{C}^i(t)$ once $n_0, n_0^*$ are met. Take player~$1$ as a fixed witness for the optimal side---any single choice works, and fixing it is what denies that term an $M$-fold gain---and define
\begin{equation}\label{eq:kstar-def}
\kappa^* := \sum_{m=1}^{T} \mathbf{1}\!\Big\{ D_{KL}(\hat{\mu}^1_{\bm{a}^*\!,m}, \tau') \leq \tfrac{2\log T}{m}\Big\},
\end{equation}
the number of pull counts $m$ of $\bm{a}^*$ at which player~$1$'s lower bound has not yet risen above $\tau'$. Each pull of $\bm{a}$ falls in one of two cases. If $n_{\bm{a}^*}(t)$ lies in the indicator set of \eqref{eq:kstar-def}, then since each cycle through $\mathcal{D}$ pulls every candidate at most once, at most one pull of $\bm{a}$ occurs for each such value, contributing at most $\kappa^*$ pulls. Otherwise player~$1$'s optimal-side condition holds, so for $\bm{a}$ to survive no player's suboptimal-side condition may fire, i.e.\ $D_{KL}(\hat{\mu}^i_{\bm{a},n_{\bm{a}}(t)}, \tau') \leq 2\log T/n_{\bm{a}}(t)$ for all $i$, adding one to $\bar\kappa$. Hence $n_{\bm{a}}(T) \leq \kappa^* + \bar{\kappa}$ on $\mathcal{G}$.

\emph{Step 3 (Bounding the counts).} The split-and-Chernoff argument of \cite[Lem.~10.8]{lattimore2020bandit}, with monotonicity on $[0,\hat\mu^1_{\bm{a}^*\!,m}]$ and the lower tail of Lemma~\ref{lem:chernoff}, gives $\E[\kappa^*] \leq 2\log T/D_{KL}(\mu^*-\delta', \mu^*-\eps_2) + 1/(2(\delta')^2)$; its $\log T$ scaling matches the leading term but its constant vanishes as $\delta' \to 0$ slowly relative to $\eps_2$. Lemma~\ref{lem:disjoint-Mfold} bounds $\E[\bar\kappa]$. Multiplying by $\Delta_{\bm{a}}$, summing over suboptimal arms, optimizing over $(\eps_1,\eps_2,\delta')$ and adding Step~1 yields \eqref{eq:probB-bound}. For the asymptotic claim take $\eps_1=\eps_2=(\log T)^{-1/2}$, $\delta' = \eps_2(\log\log T)^{-1}$: then $D_{KL}(\mu_{\bm{a}}+\eps_1, \mu^*-\eps_2) \to D_{KL}(\mu_{\bm{a}},\mu^*)$, the optimal-side term is a $(\log\log T)^{-2}$ fraction of the leading one, and the $(\delta')^{-2}$ tail is $o(\log T/D_{KL}(\mu_{\bm{a}},\mu^*))$.
\end{proof}

The sub-Gaussian analogue in \cite{chang2023optimal} replaces $D_{KL}(\mu_{\bm{a}}+\eps_1, \mu^*-\eps_2)$ by $(\Delta_{\bm{a}}-\eps_1-\eps_2)^2/2$, strictly smaller by Pinsker: at $\mu_{\bm{a}}=0.1$, $\mu^*=0.2$, $\eps_1=\eps_2=0.02$ these are $0.0227$ versus $0.0018$, a $12.6\times$ ratio, and the $1/M$ term improves analogously. The asymptotic constant exceeds the Lai--Robbins bound by a factor of two, the price of the fixed threshold (Remark~\ref{rem:threshold}). Table~\ref{tab:comparison} summarizes the three settings.

\begin{table}[t]
\caption{Leading Regret Terms: Sub-Gaussian vs.\ KL-Based Bounds}
\label{tab:comparison}
\centering
\footnotesize
\renewcommand{\arraystretch}{1.4}
\begin{tabular}{|c|c|c|}
\hline
\tstrut & \textbf{Sub-Gaussian} \cite{chang2022online, chang2023optimal} & \textbf{KL (ours)} \\
\hline
\tstrut\textbf{A} & $\displaystyle\frac{2\Delta_{\bm{a}}\log T}{(\Delta_{\bm{a}}-\eps)^2}$ & $\displaystyle\frac{\Delta_{\bm{a}}\log T}{D_{KL}(\mu_{\bm{a}}+\eps, \mu^*)}$ \\
\hline
\tstrut\textbf{B} & $\displaystyle\frac{4\Delta_{\bm{a}}\log T}{(\Delta_{\bm{a}}-\eps-\delta)^2}$ & $\displaystyle\frac{2\Delta_{\bm{a}}\log T}{D_{KL}(\mu_{\bm{a}}+\eps, \mu^*\!-\!\delta)}$ \\
\hline
\tstrut\textbf{C} & $t^{-K_0 \eps^2/(2\ln 2)}$ & $t^{-K_0 D_{KL}(\mu_{\bm{a}}+\eps, \mu_{\bm{a}})/(2\ln 2)}$ \\
\hline
\end{tabular}
\end{table}

\section{Problem C: Asymmetry in Both}\label{sec:probC}

When neither actions nor rewards are shared, elimination fails because deviations cannot be detected. We instead use deterministic sequencing of exploration and exploitation (DSEE) \cite{chang2022online}: players agree beforehand on an increasing $f: \N \to \N$ with $f(\lambda) \to \infty$ and on an ordering of joint arms. That ordering makes every player pull the same sequence of arms while exploring, so each knows which arm produced each reward; while exploiting, each commits to the arm it believes best, and one disagreement makes the joint action suboptimal. Only exploration samples are used for selection, since during exploitation the players may play different joint actions and, actions being unobservable, cannot attribute the rewards they see. Explorations start at powers of two to keep players synchronized, and the setting needs a unique optimal joint action: if two tie, players commit to different ones with constant probability at every phase, giving linear regret \cite{chang2022online}.

\begin{algorithm}[t]
\caption{\texttt{mKL-DSEE} for Problem C}\label{alg:probC}
\small
\begin{algorithmic}[1]
\REQUIRE Joint arms $[K]^M$, increasing function $f$, init.\ $\lambda = 1$.
\FOR{$t < T$}
  \IF{$t = 2^k$ for some $k \in \N$}
    \STATE \textbf{Explore:} Each player pulls every joint arm $f(\lambda)$ times using the agreed ordering.
    \STATE Each player $i$ selects $\bm{a}^i = \argmax_{\bm{a}} \hat{\mu}^i_{\bm{a}}$ using all exploration samples and plays component $\bm{a}^i[i]$.
    \STATE $\lambda \gets \lambda + 1$.
  \ELSE
    \STATE \textbf{Exploit:} Each player plays $\bm{a}^i[i]$.
  \ENDIF
\ENDFOR
\end{algorithmic}
\end{algorithm}

\begin{theorem}\label{thm:probC}
The regret of Algorithm~\ref{alg:probC} satisfies
\begin{align}\label{eq:probC-bound}
R_T &\leq \sum_{\bm{a}} \Delta_{\bm{a}}\, f\!\big(\lfloor\log_2 T\rfloor\big) \lceil\log_2 T\rceil \notag \\
&\quad + M\sum_{\bm{a}} \Delta_{\bm{a}} \sum_{t=1}^{\infty} \Big( t^{-\frac{K_0(t) D_{KL}(\mu_{\bm{a}}-\eps,\, \mu_{\bm{a}})}{2\ln 2}} \notag\\
&\qquad\qquad\qquad\quad + t^{-\frac{K_0(t) D_{KL}(\mu_{\bm{a}}+\eps,\, \mu_{\bm{a}})}{2\ln 2}} \Big),
\end{align}
where $K_0(t) = \sum_{j=1}^{\lfloor \log_2 t \rfloor} f(j)$ is the number of exploration samples of each arm by round~$t$.
\end{theorem}

\begin{proof}[Proof sketch]
The argument follows \cite{chang2022online} with KL exponents in place of sub-Gaussian ones. At most $\lceil \log_2 T \rceil$ phases each pull every joint arm $f(\lambda)$ times, giving the first term. In an exploitation round $t$ every player holds $K_0(t)$ samples of each arm, and the joint action is suboptimal only if $\hat{\mu}^i_{\bm{a}} \geq \hat{\mu}^i_{\bm{a}^*}$ for some $i$ and some $\bm{a} \neq \bm{a}^*$, which for $\eps \in (0,\Delta_{\bm{a}}/2)$ implies $\hat{\mu}^i_{\bm{a}} > \mu_{\bm{a}}+\eps$ or $\hat{\mu}^i_{\bm{a}^*} < \mu^*-\eps$; Lemma~\ref{lem:chernoff} bounds each by $e^{-K_0(t)D_{KL}(\mu\pm\eps,\,\mu)}$. A union bound over the $M$ players and over an epoch of length at most $2^\lambda \leq 2t$, on which $K_0(t)$ is constant, together with $e^{-K_0(t)c} \leq t^{-K_0(t)c/(2\ln 2)}$ for $t \geq 4$, gives the second term. By Lemma~\ref{lem:kl-mono}(c) these exponents dominate the sub-Gaussian exponent $K_0(t)\eps^2$, with ratio tending to $\frac{1}{2\mu(1-\mu)}$ as $\eps \to 0$ by Taylor expansion of the Bernoulli variance---the source of the empirical gain. The series converges once the exponent exceeds one: with $f(\lambda)=\lambda$, $K_0(t) = \Theta((\log_2 t)^2)$ and the bound is independent of $T$.
\end{proof}

Under sub-Gaussian rewards \cite{chang2022online} replaces $D_{KL}(\mu_{\bm{a}} \pm \eps, \mu_{\bm{a}})$ by $\eps^2$ here, so by Pinsker our exponents are strictly larger and the series converges with fewer exploration samples. Since $f$ trades exploitation error against exploration cost, the faster decay means a slower-growing $f$ suffices for the same error; the optimal $f$ still depends on the unknown gap $\min_{\bm{a}}\Delta_{\bm{a}}$, an unavoidable cost of having no shared information.

\section{Experiments}\label{sec:experiments}

We simulate all three settings with $T=100{,}000$, averaged over 10 runs, with all means in $[0.02,0.15]$ to target the case where KL bounds help most. Problem A uses 9 joint arms with means $(0.05, 0.03, 0.08, 0.04, 0.15, 0.06, 0.07, 0.02, 0.10)$; Problems B and C use $M=3$ players over 4 joint arms with means $(0.05, 0.03, 0.08, 0.04)$.

\begin{figure}[t]
\centering
\includegraphics[width=0.84\columnwidth]{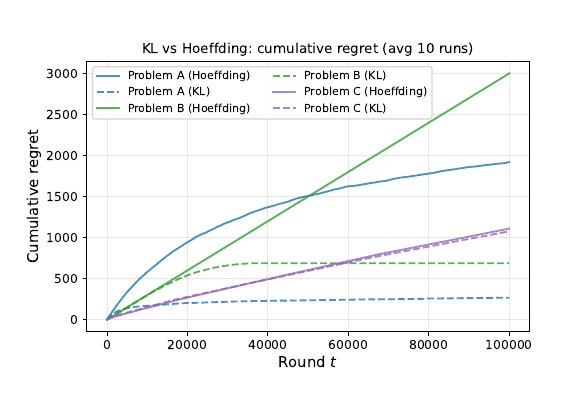}
\caption{Cumulative regret (solid: Hoeffding baseline; dashed: KL algorithm). Average over 10 runs, $T=100{,}000$.}
\label{fig:regret}
\end{figure}

The variants differ in one place each: line~3 of Algorithm~\ref{alg:probA}, where the baseline uses $\hat\mu_{\bm{a}}(t) + \sqrt{2\log f(t)/n_{\bm{a}}(t)}$ instead of \eqref{eq:klucb}, and the endpoints of \eqref{eq:int-B} in line~4 of Algorithm~\ref{alg:probB}, where it uses $\hat\mu^i_{\bm{a}}(t) \pm \sqrt{2\log T/n_{\bm{a}}(t)}$. No confidence bound appears in Algorithm~\ref{alg:probC}, so both Problem C variants run the same procedure with $f(\lambda)=\lambda$ and the bound enters only through the analysis, which is why those curves nearly coincide. Means, seeds and orderings are identical throughout.

The findings track the theory in Fig.~\ref{fig:regret}. \texttt{mKL-UCB} cuts Problem A regret by $\approx 7.6\times$ ($1{,}900 \to 250$), close to the predicted $D_{KL}(0.05,0.15)/(\Delta^2/2) \approx 9.1$. \texttt{mKL-UCB-Intervals} plateaus near $700$ on Problem B once the intervals separate, while the Hoeffding variant grows nearly linearly past $3{,}000$---a direct illustration of the threshold $n_0$ of Lemma~\ref{lem:disjoint-pair} and Remark~\ref{rem:disjoint-vs-hoeff}. Problem C shows only a marginal gap, consistent with \eqref{eq:probC-bound}: the KL exponent buys a smaller exploration schedule, but at a fixed schedule both pay the same exploration cost, which dominates at this horizon.

\textbf{Where the gap closes.} The factor $\frac{1}{2\mu(1-\mu)}$ is minimized at $\mu=1/2$, where Pinsker is tight and it equals $2$; it stays within $20\%$ of that floor for $\mu \in [0.30,0.70]$ and exceeds $5.6$ once $\mu<0.1$ or $\mu>0.9$. In Problems A and B it multiplies the leading $\log T$ term, so we expect clearly separated curves for means outside $[0.3,0.7]$ and near-indistinguishable ones inside it; in Problem C it sits in a tail exponent, so the visible gap is smaller at every mean.

\section{Discussion and Conclusion}\label{sec:conclusion}

We specialized the framework of \cite{chang2022online, chang2023optimal} to Bernoulli rewards via \texttt{mKL-UCB}, \texttt{mKL-UCB-Intervals} and \texttt{mKL-DSEE}, each improving on the sub-Gaussian baseline by a factor governed by Pinsker's inequality, with the largest gains for extreme means. The centerpiece is Problem B: a deterministic condition for two KL intervals to separate (Lemma~\ref{lem:disjoint-pair}) plus an $M$-fold gain (Lemma~\ref{lem:disjoint-Mfold}) from independent rewards.

\textbf{Why Bernoulli.} Our arguments use two structural facts: the Chernoff exponent is a KL divergence (Lemma~\ref{lem:chernoff}), and that divergence is convex and monotone in each argument (Lemma~\ref{lem:kl-mono}(a)). Both hold for any one-parameter exponential family \cite{cappe2013kullback}, so Lemma~\ref{lem:disjoint-pair} and Theorems~\ref{thm:probA}--\ref{thm:probB} carry over verbatim. What does not carry over is the size of the gain: the factor is $1/(2V(\mu))$ with $V$ the variance function, which is large only when $V$ is small relative to the sub-Gaussian variance proxy. For Bernoulli, $V(\mu)=\mu(1-\mu)$ collapses at the boundary, exactly where the motivating applications live; for families whose variance stays bounded away from zero the specialization buys at most a constant. Bernoulli is thus where the payoff is largest.

\textbf{Anytime elimination.} Remark~\ref{rem:threshold} is not an artifact of the analysis. Elimination is irreversible, so one wrong removal of $\bm{a}^*$ costs $\Theta(T)$ and the probability that any interval ever excludes its mean must be $O(1/T)$, whereas an anytime threshold of order $\log t$ leaves a failure probability of order $t^{-2}$ at round $t$, which sums to a constant. Doubling does recover one: run Algorithm~\ref{alg:probB} on blocks of length $2^r$ with threshold $2\log 2^r$ and $\mathcal{D}$ reset at each block, so each contributes $O(r/D_{KL})$ and the total is $O(\log^2 T)$. Whether the extra $\log T$ is necessary, or a rule allowing occasional reinstatement attains $O(\log T)$ without knowing $T$, is open.

\textbf{Tightness of $1/M$.} That the $M$-fold gain sits in the additive term is structural: all players pull the same joint arm each round, so the leading $\log T$ term is set by the samples one player needs for its own intervals to separate (Lemma~\ref{lem:disjoint-pair}), not by the $Mn_{\bm{a}}(t)$ samples the group holds. Pooling those samples would divide the leading term by $M$, and pooling is what the absence of communication forbids. Algorithm~\ref{alg:probB} does carry one bit from each player each round through the deviation channel; whether that channel can aggregate estimates without breaking coordination is the more promising form of the question.

\bibliographystyle{IEEEtran}
\bibliography{references}

\end{document}